%% file: main.tex
\documentclass[sigconf]{acmart}
\copyrightyear{2026}
\acmYear{2026}
\setcopyright{cc}
\setcctype{by}
\acmConference[KDD 2026] {Proceedings of the 32nd ACM SIGKDD Conference on Knowledge Discovery and Data Mining V.2}{August 9--13, 2026}{Jeju Island, Republic of Korea.}
\acmBooktitle{Proceedings of the 32nd ACM SIGKDD Conference on Knowledge Discovery and Data Mining V.2 (KDD 2026), August 9--13, 2026, Jeju Island, Republic of Korea}
\acmISBN{979-8-4007-2259-2/2026/08}
\acmDOI{10.1145/3770855.3817727}
\usepackage{algorithm}
\usepackage{algorithmic}
\usepackage{booktabs}
\usepackage{xcolor} 
\usepackage{pifont}
\usepackage{amsmath}
\usepackage{amsthm}
\usepackage{amsfonts}
\usepackage{multirow}
\usepackage{marvosym}
\usepackage{balance}
\usepackage{newfloat}
\usepackage{listings}
\usepackage{colortbl}
\definecolor{headerbg}{HTML}{FFFFFF} 
\definecolor{body1}  {HTML}{FFFFFF} 
\definecolor{body2}  {HTML}{FFFFFF} 
\newtheorem{theorem}{Theorem}[section]
\AtBeginDocument{%
  }

\begin{document}

\newcommand{\method}[0]{AlphaEval}
\title{\method: A Comprehensive and Efficient Evaluation Framework for Formula Alpha Mining}

\author{Hongjun Ding}
\authornote{Both authors contributed equally to this research.}
\affiliation{%
  \institution{CUNY Baruch College}
  \city{New York}
  \country{USA}
}
\email{hongjun.ding.baruchmfe@gmail.com}

\author{Binqi Chen}
\authornotemark[1]
\authornote{State Key Laboratory for Multimedia Information Processing, School of Computer Science, PKU-Anker LLM Lab, Beijing Key Laboratory of Software and Hardware Cooperative Artificial Intelligence Systems, Peking University, Beijing, China}
\affiliation{
  \institution{Peking University}
  \city{Beijing}
  \country{China}
}
\email{cbq@stu.pku.edu.cn}
\author{Jinsheng Huang}
\authornotemark[2]
\affiliation{
  \institution{Peking University}
  \city{Beijing}
  \country{China}
}
\email{hjs@stu.pku.edu.cn}
\author{Taian Guo}
\authornotemark[2]
\affiliation{
  \institution{Peking University}
  \city{Beijing}
  \country{China}
}
\email{taianguo@stu.pku.edu.cn}
\author{Zhengyang Mao}
\authornotemark[2]
\affiliation{
  \institution{Peking University}
  \city{Beijing}
  \country{China}
}
\email{zhengyang.mao@stu.pku.edu.cn}

\author{Guoyi Shao}
\affiliation{%
  \institution{Peking University}
  \city{Beijing}
  \country{China}
}
\email{2100012950@stu.pku.edu.cn}

\author{Lutong Zou}
\affiliation{%
  \institution{Harvard University}
  \state{Massachusetts}
  \city{Cambridge}
  \country{USA}
}
\email{xjqrxjqr@gmail.com}

\author{Luchen Liu}
\affiliation{
\institution{Zhengren Research, Zhengren Quant}
\city{Haikou}
\state{Hainan}
  \country{China}
}
\email{liulc@zhengrenquant.com}
\author{Ming Zhang}
\authornotemark[2]
\authornote{Corresponding Author}
\affiliation{
  \institution{Peking University}
  \city{Beijing}
  \country{China}
}
\email{mzhang\_cs@pku.edu.cn}

\renewcommand{\shortauthors}{Hongjun Ding et al.}

\input{tex/0_abs}




\maketitle
\input{tex/1_intro}
\input{tex/2_relate}
\input{tex/3_problem}
\input{tex/4_method}
\input{tex/5_result}
\input{tex/6_conlusion}

\section*{Acknowledgments}
This paper is partially supported by grants from the National Key Research and Development Program of China with Grant No. 2023YFC-3341203 and the National Natural Science Foundation of
China (NSFC Grant Number 62276002).
\clearpage

\bibliographystyle{ACM-Reference-Format}
\balance
\bibliography{cite}

\appendix
\input{tex/7_appendix}

\end{document}

%% file: tex/0_abs.tex
\begin{abstract}
Formula alpha mining, which generates predictive signals from financial data, is critical for quantitative investment. Although various algorithmic approaches—such as genetic programming, reinforcement learning, and large language models—have significantly expanded the capacity for alpha discovery, systematic evaluation remains a key challenge. Existing evaluation metrics predominantly include backtesting and correlation-based measures. Backtesting is computationally intensive, inherently sequential, and sensitive to specific strategy parameters. Correlation-based metrics, though efficient, assess only predictive ability and overlook other crucial properties such as temporal stability, robustness, diversity, and interpretability. Additionally, the closed-source nature of most existing alpha mining models hinders reproducibility and slows progress in this field. To address these issues, we propose \textbf{AlphaEval}, a unified, parallelizable, and backtest-free evaluation framework for automated alpha mining models. AlphaEval assesses the overall quality of generated alphas along five complementary dimensions: predictive power, stability, robustness to market perturbations, financial logic, and diversity. Extensive experiments across representative alpha mining algorithms demonstrate that AlphaEval achieves evaluation consistency comparable to comprehensive backtesting, while providing more comprehensive insights and higher efficiency. Furthermore, AlphaEval effectively identifies superior alphas compared to traditional single-metric screening approaches. All implementations and evaluation tools are open-sourced to promote reproducibility and community engagement.
\end{abstract}

%% file: tex/1_intro.tex
\section{Introduction}
\label{sec:intro}

The automated mining of formula alpha is a central challenge in quantitative investment. Formula alpha\footnote{This paper subsequently uses \textbf{alpha} to refer to \textbf{formula alpha}.}, defined as computable expressions that transform raw financial data into signals predictive of future returns, have evolved from handcrafted models grounded in financial theory~\cite{fama_2015_five} to large-scale automated discovery. Recent developments include genetic programming~\cite{chen_2021_gp,zhang_autoalpha_2020, cui_alphaevolve_2021}, reinforcement learning~\cite{alphagen,xu_textalpha2_2024,zhao_quantfactor_2024,ren_riskminer_2024, zhu_alphaqcm,zhao2025learningexpertfactorstrajectorylevel}, generative adversarial networks~\cite{shi_alphaforge_2025}, and large language models (LLMs)~\cite{li-2024-fama,chainofalpha,tang2025alphaagentllmdrivenalphamining,ren_linear_2025,rdagentquant}, enabling the generation of vast numbers of candidate alphas.

A summary of representative alpha mining models is shown in Table~\ref{tab:summary}. While many of these methods demonstrate promising results in alpha generation, their evaluation schemes are often limited, inconsistent, and incomplete. In practice, two types of evaluation are commonly used: backtesting and correlation-based metrics such as the Information Coefficient (IC) or RankIC~\cite{wang_2025_quantbench}. Backtesting simulates portfolio performance using historical market data, but it is inherently sequential, computationally expensive, and highly sensitive to strategy design choices. IC-based metrics provide a lightweight proxy for assessing the linear correlation between alphas and future returns, yet they focus solely on predictive ability and fail to capture other essential dimensions of alpha quality—such as temporal stability, robustness to market perturbations, diversity, and logical interpretability. These limitations make it difficult to perform fair and comprehensive comparisons across mining models, especially in contexts where backtesting strategy is different. Furthermore, most alpha mining models remain closed-source, which hinders reproducibility and slows progress in this important area of quantitative research.

\begin{table}[ht]
    \centering
    \caption{Summary of current alpha mining models. Metrics in \textbf{bold} are based on backtesting.}
    \rowcolors{2}{body1}{body2}
    \begin{tabular}{ccc}
        \rowcolor{headerbg}
        \toprule
        \textbf{Method} & \textbf{Metrics} & \textbf{Code} \\
        \midrule
        GP~\cite{chen_2021_gp} & \textbf{AR} & {\color{red}\ding{55}} \\
        AutoAlpha~\cite{zhang_autoalpha_2020} & IC, \textbf{AR}, \textbf{SR} & {\color{red}\ding{55}} \\
        AlphaEvolve~\cite{cui_alphaevolve_2021} & IC, \textbf{SR} & {\color{red}\ding{55}} \\
        AlphaGen~\cite{alphagen} & IC, RankIC, \textbf{AR} & {\color{teal}\ding{51}} \\
        RiskMiner~\cite{ren_riskminer_2024} & IC, ICIR, RankIC & {\color{red}\ding{55}} \\
        QFR~\cite{zhao_quantfactor_2024} & IC, RankIC & {\color{red}\ding{55}} \\
        $\text{Alpha}^2$~\cite{xu_textalpha2_2024} & IC & {\color{red}\ding{55}} \\
        AlphaForge~\cite{shi_alphaforge_2025} & IC & {\color{teal}\ding{51}} \\
        AlphaQCM~\cite{zhu_alphaqcm} & IC & {\color{teal}\ding{51}} \\
        FAMA~\cite{li-2024-fama} & RankIC, RankICIR & {\color{red}\ding{55}} \\
        AlphaAgent~\cite{tang2025alphaagentllmdrivenalphamining} & IC, ICIR, \textbf{AR}, IR, \textbf{MaxDD} & {\color{teal}\ding{51}} \\
        \bottomrule
    \end{tabular}
    \label{tab:summary}
\end{table}

To address this gap, we propose \textbf{AlphaEval}, a structured and efficient evaluation framework for automated alpha mining models. Unlike traditional approaches, AlphaEval evaluates an alpha mining model holistically based on the collection of alphas it produces—without requiring portfolio backtesting. Our framework scores models along five complementary dimensions: \textit{predictive power}, \textit{temporal stability}, \textit{robustness to market perturbations}, \textit{financial logic}, and \textit{diversity}. These metrics are designed to be parallelizable, interpretable, and applicable across different models and markets.

We apply AlphaEval to a suite of representative alpha mining models. The results show that AlphaEval scores are highly consistent with precision backtesting outcomes while offering broader diagnostic insight and significantly faster evaluation. Furthermore, AlphaEval demonstrates superior alpha selection performance compared to conventional single-metric filtering approaches (e.g., by IC alone).

In summary, our contributions are as follows:
\begin{itemize}
    \item We propose \textbf{AlphaEval}, the first unified, backtest-free, and parallelizable framework for evaluating automated alpha mining models.
    \item We design five complementary metrics that comprehensively assess the predictive quality, temporal stability, robustness, interpretability, and diversity of generated alphas.
    \item We conduct large-scale benchmarking across eight popular mining models, showing AlphaEval’s effectiveness and consistency with traditional backtesting.
    \item We open-source all implementations and evaluation tools to foster transparency and reproducibility in the quantitative finance community.
\end{itemize}

%% file: tex/2_relate.tex
\section{Related Work}
\label{sec:related}

\subsection{Alpha Mining}

Alpha mining typically consists of two sequential stages: first, generating candidate alpha set, and second, selecting and combining these alphas into a predictive signal.

In the alpha generation stage, early research primarily focused on manually crafting alphas based on economic insights or empirical patterns, exemplified by classical models such as Fama-French~\cite{fama_2015_five} and curated alpha libraries like Alpha101~\cite{kakushadze_2016_alpha101} and Alpha158~\cite{yang_2020_qlib}. Although these handcrafted alphas are intuitive and interpretable, their diversity and expressiveness are inherently limited. To address these limitations, researchers introduced automated approaches such as genetic algorithm (GA)~\cite{chen_2021_gp,zhang_autoalpha_2020, cui_alphaevolve_2021}, reinforcement learning (RL)~\cite{alphagen, zhu_alphaqcm}, generative adversarial networks (GANs)~\cite{shi_alphaforge_2025} and more recently, large language models (LLMs)~\cite{li-2024-fama,tang2025alphaagentllmdrivenalphamining}. Evolutionary algorithms and RL-based methods systematically explore large symbolic search spaces, generating a substantial number of potential alphas; however, these automated methods frequently produce factors lacking clear financial interpretability and often face challenges in generalizing across different market conditions. LLM-based approaches leverage financial linguistic understanding to generate more interpretable and semantically meaningful alpha expressions, but often lack comprehensive validation of their practicality due to simplified evaluation criteria.

In the alpha selection and combination stage, generated alpha candidates are assessed and integrated into predictive strategies. A common workflow involves applying metric-based thresholding, such as Information Coefficient (IC) or Sharpe ratio, to initially filter promising alphas, followed by modeling their relationships through techniques including linear regression, LightGBM~\cite{ke_lightgbm_2017}, or XGBoost~\cite{chen_xgboost_2016}. An alternative strategy employed by RL-based methods (e.g., AlphaGen, AlphaQCM) directly incorporates the selection and combination tasks into the alpha discovery process itself. These methods optimize portfolio-level performance metrics as rewards during alpha search, thereby creating an end-to-end optimization framework. Although this integrated optimization strategy shows promise in enhancing alpha quality, it increases model complexity and computational cost, requiring stronger supervision signals and longer training cycles.

\subsection{Alpha Evaluation Metrics}

The dominant evaluation metrics in alpha mining focus on assessing predictive power. Information Coefficient (IC), RankIC, and various return-based metrics (e.g., annual return, Sharpe ratio) are widely used to quantify the association between alpha signals and future returns. While effective in measuring short-term predictability, these metrics offer a narrow view of alpha quality. They fail to account for critical properties such as stability over time, robustness to market noise, structural diversity among alphas, and logical consistency. Moreover, backtesting—a common but expensive evaluation method—introduces sensitivity to strategy configurations, suffers from low parallelism, and limits its scalability in large-scale alpha generation tasks.

In addition, current evaluation protocols typically treat each alpha independently, lacking a mechanism to assess the collective performance or quality of the alpha set generated by a mining model. This hinders fair comparison between models and undermines efforts to understand their generalization capabilities. A more holistic evaluation framework is urgently needed to support model-level diagnosis and efficient alpha selection.

\subsection{Interpretability and Robustness in Quantitative Models}

In practical financial applications, interpretability and robustness are increasingly viewed as essential for risk management, regulatory compliance, and human-in-the-loop decision making~\cite{rudin_2019_explaining, tatsat_2025_interpretability}. Formulaic alphas—unlike deep learning-based signals—have the advantage of being inherently transparent and interpretable. Yet, few existing benchmarks incorporate metrics that explicitly reward logical clarity or penalize unstable behavior.

Recent work in interpretable machine learning has emphasized model transparency~\cite{rane_2023_explainable,ravi_2025_explainable} and behavioral consistency~\cite{chen_2023_can,soton_2025_enhancing}, but their application in alpha mining remains limited. Likewise, robustness to noise and temporal perturbations is rarely examined in evaluating alpha quality, despite being critical for real-world deployment.

Our work fills these gaps by proposing \textbf{AlphaEval}, a unified evaluation framework that integrates interpretability, stability, and robustness into the assessment of alpha mining algorithms. In doing so, we shift the focus from narrow, label-dependent metrics to a more comprehensive, model-level evaluation approach.

%% file: tex/3_problem.tex
\section{Definition \& Preliminary}
\label{sec:def}

Let $X \in \mathbb{R}^{T \times N \times F}$ denote a panel of financial features, where $T$ is the number of time steps, $N$ is the number of assets (e.g., stocks), and $F$ is the number of features per asset. Correspondingly, let $y \in \mathbb{R}^{T \times N}$ denote the future returns:
\begin{equation}
\label{eq:def_y}
y_{t,n}^{\Delta T}=\frac{\text{close}_{t+\Delta T,n}-\text{close}_{t,n}}{\text{close}_{t,n}},
\end{equation}
where $\text{close}_{t,n}$ denotes the closing price of stock $n$ at time $t$, $\Delta T$ denotes the prediction interval\footnote{If not specified then $\Delta T = 1$.}.

The goal of alpha mining is to construct a set of alphas $\mathcal{A} = \{\alpha_i\}_{i=1}^{K}$, where each $\alpha_i$ is a symbolic or parametric function that produces a score matrix $S^{(i)} \in \mathbb{R}^{T \times N}$:
\begin{equation}
\label{eq:def_alpha}
S^{(i)}_{t,n} = \alpha_i\left( X_{t-L^{(i)}+1:t,n,:} \right),
\end{equation}
where $X_{t-L^{(i)}+1:t,n,:} \in \mathbb{R}^{L^{(i)} \times F}$ denotes the sequence of past $L^{(i)}$ feature vectors of asset $n$ up to time $t$. Each alpha thus operates on a temporal slice of the data and outputs a scalar score per asset and time.

\paragraph{Stage I: Alpha Generation.}
The first stage aims to discover candidate alphas from the data via automated procedures. These include symbolic regression (e.g., genetic programming), reinforcement learning, or language models. The outcome is a candidate pool:
\begin{equation}
\mathcal{A}_{\text{gen}} = \{\alpha_1, \alpha_2, \ldots, \alpha_K\}, \quad \alpha_i: \mathbb{R}^F \rightarrow \mathbb{R}.
\end{equation}
Each $\alpha_i$ maps asset-level features to scalar scores over the entire panel $X$, resulting in a matrix $S^{(i)} \in \mathbb{R}^{T \times N}$.

\paragraph{Stage II: Alpha Selection and Combination.}
Given the candidate pool $\mathcal{A}_{\text{gen}}$, the second stage selects a subset $\mathcal{A}_{\text{sel}} \subseteq \mathcal{A}_{\text{gen}}$ and combines them into a final signal matrix $S \in \mathbb{R}^{T \times N}$:
\begin{equation}
S_{t,n} = \mathcal{F} \left( \{\alpha_i(X_{t-L^{(i)}+1:t,n,:})\}_{\alpha_i \in \mathcal{A}_{\text{sel}}} \right),
\end{equation}
where $\mathcal{F}$ denotes a combination function, such as a weighted linear combination or a nonlinear model like LightGBM~\cite{ke_lightgbm_2017} or XGBoost~\cite{chen_xgboost_2016}.

%% file: tex/4_method.tex
\section{AlphaEval}
\label{sec:alphaeval}

To enable efficient and comprehensive evaluation of alpha mining models, we introduce \textbf{AlphaEval}, a multi-dimensional assessment framework that quantifies the quality of both the generated alpha signals and the underlying mining algorithms. 

Unlike traditional evaluation paradigms focused solely on predictive metrics such as IC or backtest returns, AlphaEval offers a unified benchmark across five complementary dimensions: \textit{predictive power}, \textit{temporal stability}, \textit{robustness to market perturbations}, \textit{financial logic}, and \textit{diversity}. Among them, the first four dimensions are evaluated for the alpha quality and the last one dimension is evaluated for the mining ability of the model. An overview of AlphaEval is shown in Figure \ref{fig:main}.
\begin{figure*}[ht]
    \centering
    \includegraphics[width=1.0\linewidth]{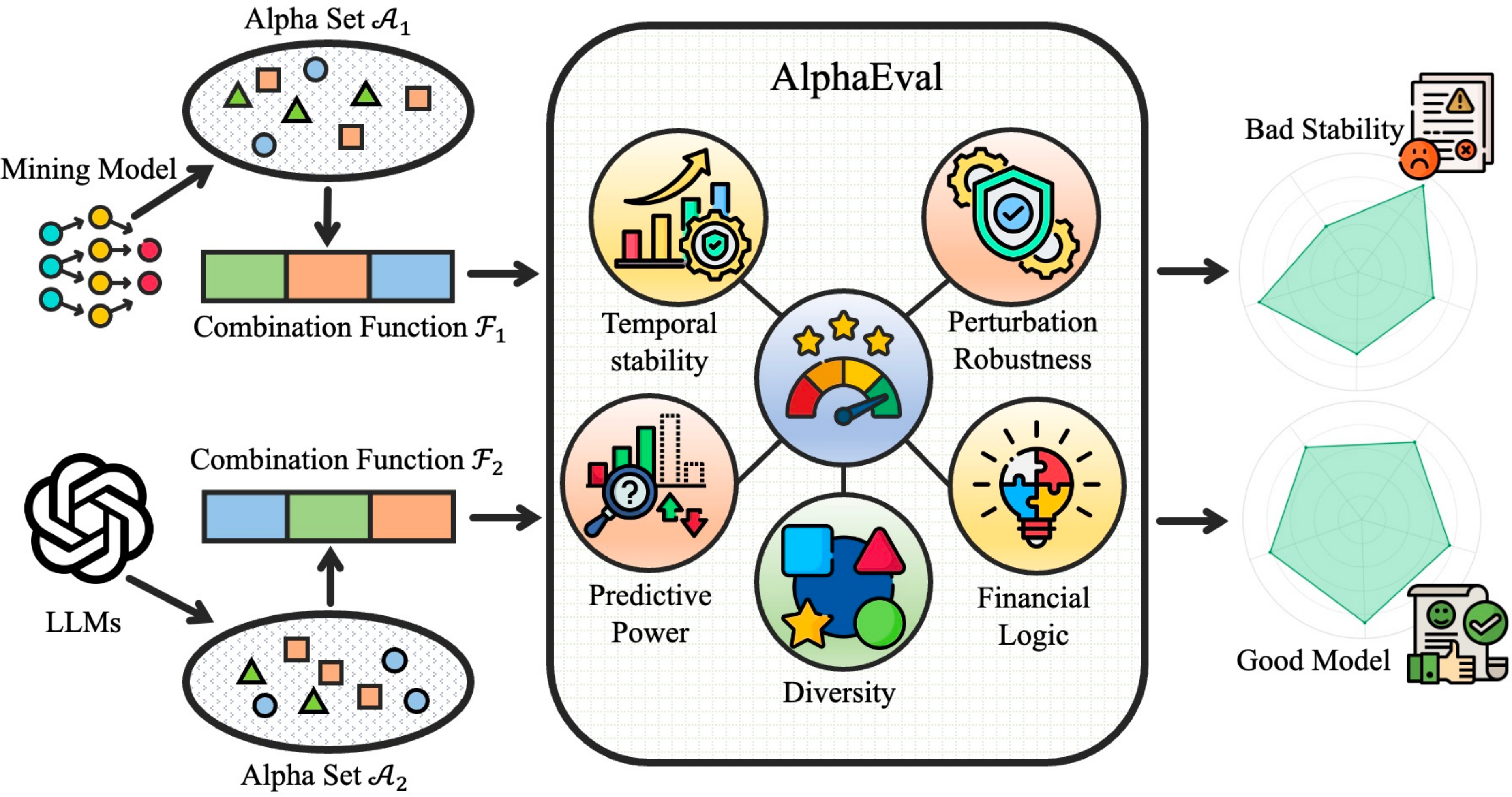}
    \caption{Overview of AlphaEval. After getting the alpha calculation results from the alpha mining model, the combined alphas are obtained by the alpha combination model and evaluated by AlphaEval, which can get the scores of different dimensions for a comprehensive evaluation of the model.}
    \label{fig:main}
\end{figure*}
In the following, we detail the motivation and implementation of each evaluation dimension.

\subsection{Predictive Power}
\label{sec:predpower}

The most fundamental property of an alpha is its ability to predict future returns. In AlphaEval, we retain this classical perspective and include predictive power as one core dimension.

We adopt two widely used correlation-based metrics:

\begin{itemize}
    \item \textbf{Information Coefficient (IC)}: Defined as the average Pearson correlation between the alpha scores $S_{t,:}$ and the realized returns $y_{t,:}$ across assets over all time steps:
    \begin{gather}
    \label{eq:ic}
        \mathrm{IC} = \frac{1}{T} \sum_{t=1}^{T}\mathrm{IC}_t,\\
        \mathrm{IC}_t=\frac{\sum_{i=1}^N\left(S_{t,i}-\bar{S_{t,:}}\right)\left(y_{t,i}-\bar{y_{t,:}}\right)}{\sqrt{\sum_{i=1}^N\left(S_{t,i}-\bar{S_{t,:}}\right)^2}\sqrt{\sum_{i=1}^N\left(y_{t,i}-\bar{y_{t,:}}\right)^2}}.
    \end{gather}

    \item \textbf{Rank Information Coefficient (RankIC)}: Defined as the average Spearman rank correlation between $S_{t,:}$ and $y_{t,:}$ across time:
    \begin{gather}
    \label{eq:rankic}
        \mathrm{RankIC} = \frac{1}{T}\mathrm{RankIC}_t,\\ \mathrm{RankIC}_t=\sum_{t=1}^{T} \left(1-\frac{6\sum_{i=1}^Nd_i^2}{N\left(N^2-1\right)}\right),\\
        d_i\equiv R\left[S_{t,i}\right]-R\left[y_{t,i}\right],
    \end{gather}
    where $R\left[\cdot_{t,i}\right]$ denotes the rank of $\cdot_{t,i}$ in $\cdot_{t,:}$.
\end{itemize}

Based on these metrics, we propose the \textbf{Predictive Power Score (PPS)}, defined as follows:
\begin{equation}
    \mathrm{PPS}=\beta\cdot\mathrm{IC}+(1-\beta)\cdot\mathrm{RankIC},
\end{equation}
where $\beta$ is a hyperparameter that controls the IC and RankIC occupancy. The metric summarize the predictive strength of an alpha across time. A higher PPS indicates stronger alignment between the alpha scores and subsequent asset returns, which is crucial for investment decision-making.

\subsection{Temporal Stability}
\label{sec:temporal_stability}

While high predictive accuracy is desirable, unstable alphas are difficult to deploy in real trading environments. Temporal stability measures how consistent an alpha's ranking of assets remains across consecutive time steps, reflecting its reliability over time.

To quantify this, we propose the \textbf{Relative Rank Entropy (RRE)}:

\begin{equation}
    \mathrm{RRE} = \frac{1}{T-1} \sum_{t=2}^{T} \frac{1}{1+\mathrm{KL}\left( S_t \| S_{t-1} \right)},
\end{equation}

where $\mathrm{KL}(p \| q)$ is a divergence-based entropy between two rank vectors at time $t$ and $t-1$. In practice, we compute the KL dispersion by converting the ranked arrangement into a discrete distribution:
\begin{gather}
    p\left(S_{t,i}\right)=\frac{R\left[S_{t,i}\right]}{\sum_{j=1}^NR\left[S_{t,j}\right]},\\
    \mathrm{KL}\left(S_t\|S_{t-1}\right)=\sum_{i=1}^Np\left(S_{t,i}\right)\log\frac{p\left(S_{t,i}\right)}{p\left(S_{t-1,i}\right)}.
\end{gather}

A higher RRE indicates greater temporal consistency in asset ranking, which is favorable for stable portfolio construction and lower turnover. This stability is a desirable property in risk-sensitive or strategy-constrained scenarios.

\subsection{Robustness to Market Perturbations}
\label{sec:robustness}

In financial markets, features are often subject to random fluctuations or structural shocks. A robust alpha should remain stable under such perturbations. To this end, we propose \textbf{Perturbation Fidelity Score (PFS)} to evaluate the sensitivity of alpha rankings to input-level noise.

Formally, let $\varepsilon \sim \mathcal{D}$ be a perturbation applied to the original feature tensor $X$, and define the perturbed alpha score as $S' = \alpha(X + \varepsilon)$. The robustness of an alpha is quantified as the correlation between the original and perturbed asset rankings:

\begin{equation}
    \mathrm{PFS}_{\mathcal{D}} = \mathrm{Corr} \left(S, S'\right),
\end{equation}

where $\text{Corr}(\cdot,\cdot)$ denotes the Spearman rank correlation, similar to the RankIC definition above.

We consider two types of perturbation distributions:
\begin{itemize}
    \item \textbf{Gaussian noise} ($\varepsilon \sim \mathcal{N}(0, \sigma^2)$): simulates random fluctuations driven by market sentiment or microstructure noise.
    \item \textbf{$t$-distribution} ($\varepsilon \sim t(\nu)$): mimics structural market shocks such as policy changes or crises, introducing heavy-tailed disturbances.
\end{itemize}
Based on these two different perturbations, the PFS is defined as follows:
\begin{equation}
    \mathrm{PFS} = \min \{\mathrm{PFS}_{\mathcal{N}(0, \sigma^2)},\mathrm{PFS}_{t(\nu)}\}.
\end{equation}
An alpha with high $\mathrm{PFS}$ is considered more robust and reliable in volatile or nonstationary market environments.

\subsection{Financial Logic}
\label{sec:logic}

In addition to statistical properties, the financial interpretability of alpha signals plays a crucial role in practical deployment, especially for risk management and compliance purposes. To assess the financial plausibility of a given alpha, we introduce a \textbf{Logic Score}, rated by a Large Language Model with financial knowledge.

Given the symbolic expression or natural language description of an alpha $\alpha_i$, we prompt an LLM to evaluate its logical coherence, economic intuition, and interpretability.

The LLM's response is parsed to extract a numerical score, which we denote as the Logic Score for $\alpha_i$. We average this score across a set of alphas to summarize the logical quality of a mining algorithm.

While inherently subjective, this mechanism reflects a growing trend of combining human-aligned reasoning with automated alpha discovery. It complements traditional metrics by incorporating domain-informed assessments that cannot be easily captured by statistical measures alone.

\subsection{Diversity}
\label{sec:diversity}

A desirable alpha set should contain diverse signals to avoid redundancy and enhance robustness when combined. We propose \textbf{Diversity Entropy (DE)}, which quantifies the diversity of the selected alpha set by analyzing the covariance structure of the output signals.

Let $\{S^{(i)}\}_{i=1}^m$ be $m$ selected alpha signals, each flattened into a vector over all $(t,n)$ pairs. Let $C \in \mathbb{R}^{m \times m}$ be the covariance matrix computed over these $m$ alpha vectors. We denote the eigenvalues of $C$ as $\lambda_1, \lambda_2, \dots, \lambda_m$.

To measure the distributional spread of variance across alpha signals, we normalize the eigenvalues into a probability distribution:

\begin{equation}
    p_i = \frac{\lambda_i}{\sum_{j=1}^m \lambda_j}.
\end{equation}

The DH is then defined as the entropy of this distribution:

\begin{equation}
    \mathrm{DH} = \frac{-\sum_{i=1}^m p_i \log p_i}{\log m}.
\end{equation}

Higher entropy indicates a more diverse alpha set that captures complementary information from multiple dimensions.
\begin{table*}[!h]
    \centering
    \caption{Performance of alpha mining models on A-Share under the AlphaEval framework. \textbf{Bold} is the highest, \underline{underlined} is the second, random is only used as a reference value and is not involved in the comparison.}
    \label{tab:main-results}
    \begin{tabular}{ccccccc}
        \toprule
        \multicolumn{2}{c}{Method} & Predictive$\uparrow$ & Stability$\uparrow$ & Robustness$\uparrow$ & Diversity$\uparrow$ & Logic$\uparrow$ \\
        \midrule
        &Random & 0.009 & 0.844 & 0.846 & 0.981 & 60.0 \\
        \midrule
        \multirow{3}{*}{GA-Based}&GP & 0.017 & 0.724 & 0.983 & 0.693 & 63.5 \\
        &AutoAlpha  & 0.027 & 0.774 & 0.971 & \textbf{0.946} & 64.0 \\
        &AlphaEvlove & 0.028 & 0.975 & 0.688 & \underline{0.897} & 63.0 \\
        \midrule
        \multirow{2}{*}{RL-Based} & AlphaGen & 0.034 & \textbf{0.978} & \textbf{0.997} & 0.650 & 59.0 \\
        &AlphaQCM & 0.029 & 0.975 & \underline{0.996} & 0.477 & 62.0 \\
        \midrule
        GANs-Based & AlphaForge & \underline{0.040} & \underline{0.977} &  0.677 & 0.743 & 62.5 \\
        \midrule
        \multirow{2}{*}{LLMs-Based}&FAMA & 0.031 & 0.868 & 0.992 & 0.831 & \underline{69.0} \\
        &AlphaAgent & \textbf{0.041} & 0.779 & 0.415 & 0.812 & \textbf{70.0} \\   
        \bottomrule
    \end{tabular}
\end{table*}

\subsection{Overall AlphaEval Score}
\label{sec:alphaeval_overall}

To aggregate the five dimensions of AlphaEval---Predictive Power (PPS), Temporal Stability (RRE), Robustness to Perturbations (PFS), Financial Logic (Logic), and Diversity (DE)---we compute a normalized, composite score for each candidate alpha $j$. For each metric $m \in \{\text{PPS},\text{RRE},\text{PFS},\text{Logic},\text{DE}\}$, we standardize scores within each dataset and evaluation round: 
$\tilde m_{j} \;=\; \frac{m_{j}-\mu_{m}}{\sigma_{m}}, $where $\mu_m$ and $\sigma_m$ are the mean and standard deviation of metric $m$ across all alphas under comparison. All five metrics are oriented so that larger values indicate better quality. The overall AlphaEval score is the convex combination 
    $\mathrm{AlphaEval}(j) \;=\; \frac{1}{m}\sum_{m} w_m\,\tilde m_{j}.$

%% file: tex/5_result.tex
\section{Experiments \& Results}
\label{sec:experiments}

In this section, we conduct comprehensive experiments to demonstrate the effectiveness and practicality of the AlphaEval framework. We aim to answer the following questions:

\begin{itemize}
    \item \textbf{Q1:} How do mainstream models perform under the evaluation of AlphaEval?
    \item \textbf{Q2:} Do the proposed evaluation dimensions offer complementary information beyond traditional metrics?
    \item \textbf{Q3:} Is the evaluation framework robust when hyperparameters are adjusted?
    \item \textbf{Q4:} Are the evaluation scores aligned with real-world investment behaviors such as turnover and drawdown?
    \item \textbf{Q5:} Does AlphaEval significantly speed up the evaluation process compared to backtesting-based evaluation systems?
\end{itemize}

\subsection{Experimental Setting}
\label{sec:setting}

All evaluations are conducted using our implementation of AlphaEval on the public Qlib platform. We use both A-share and U.S. stock datasets provided by Qlib as our evaluation benchmarks\footnote{Detailed information regarding the dataset split is provided in the Appendix~\ref{appendix:dataset}.}. For the Predictive Power Score (PPS), we set the weighting parameter $\beta$ to 0.5 to balance between predictive accuracy and stability over time. For the Perturbation Fidelity Score (PFS), we apply two types of noise: Gaussian and Student’s t-distributed. The standard deviation of Gaussian noise is set to the average daily volatility of the corresponding market index. For the t-distribution, the degrees of freedom are fixed at 3, and the distribution is rescaled to match the same standard deviation as the Gaussian case. This ensures a controlled comparison of robustness to market sentiment perturbations and policy changes perturbations.

\subsection{Main Results}
\label{sec:mainresults}

To answer \textbf{Q1}, We group models by methodology—genetic algorithm (GA-Based), reinforcement learning (RL-Based), generative adversarial networks (GANs-Based), and LLMs (LLMs-Based)—and compare them across all dimensions. Table~\ref{tab:main-results} presents the performance of all baseline models under the proposed AlphaEval framework on A-share (China market) and the results on U.S. stock dataset are provided in Appendix~\ref{appendix:er}.

GA-based methods demonstrate strong robustness and stability, with GP achieving the highest robustness (0.983) and AutoAlpha excelling in diversity (0.946), though overall interpretability remains limited. RL-based methods, particularly AlphaGen, show outstanding stability (0.978) and robustness (0.997), alongside competitive predictive power, but suffer from low logic scores, indicating poor transparency. GANs-based methods offer high predictive performance (0.040 for AlphaForge) and solid stability, but their robustness and logic consistency are less reliable. In contrast, LLMs-based methods—especially AlphaAgent—achieve the best overall trade-off: highest predictive power (0.041), best logic clarity (70.0), and strong diversity, but with slightly lower robustness. These results suggest that while RL and GA methods offer behavioral reliability and search diversity, LLMs stand out by combining high predictive accuracy with semantic interpretability, making them particularly well-suited for human-in-the-loop financial applications.

\begin{figure*}[t]
    \centering
    \includegraphics[width=0.95\linewidth]{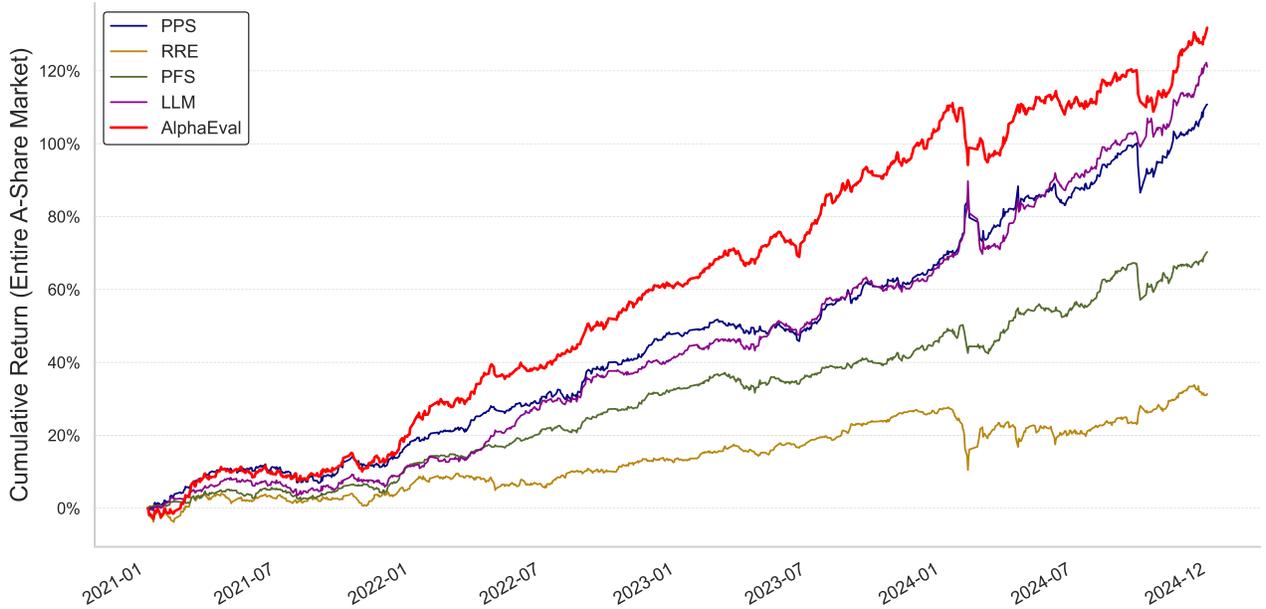}
    \caption{Cumulative returns of portfolios constructed by selecting top-ranked alpha alphas based on different evaluation metrics: PPS, RRE, PFS, LLM Logic Score, and the integrated AlphaEval score.}

    \label{fig:ablation}
\end{figure*}

\subsection{Ablation Study}
\label{sec:ablation}

To answer \textbf{Q2}, we conduct an ablation study based on metric-specific alpha selection. From the full set of candidate alphas—generated during the model search process, including those not selected for final output—we construct portfolios using top-ranked alphas according to individual metrics: \textbf{Predictive Power Score (PPS)}, \textbf{Relative Rank Entropy (RRE)}, \textbf{Perturbation Fidelity Score (PFS)}, and \textbf{LLM Logic Score}. These are compared against the integrated selection based on the full \textbf{AlphaEval} score.

Figure~\ref{fig:ablation} shows the cumulative returns on the A-share market from 2021 to 2024. We find that each individual metric contributes positively to portfolio performance, reflecting its unique perspective on alpha quality. PPS and LLM Logic yield relatively strong returns, highlighting their effectiveness in capturing predictive and semantic strength, respectively. However, when used alone, they occasionally suffer from instability or diminished robustness. In contrast, RRE and PFS provide more conservative but stable profiles, emphasizing structural stability and noise resistance.

Importantly, the full AlphaEval score, which integrates all metrics, consistently outperforms any single-metric selection. This confirms that the proposed dimensions capture \textit{complementary} aspects of alpha quality—such as predictability, robustness, diversity, and interpretability—and their combination leads to a more reliable and effective evaluation signal than traditional metrics alone.

\subsection{Sensitivity Analysis}
\begin{figure*}[!t]
    \centering
    \includegraphics[width=1.\linewidth]{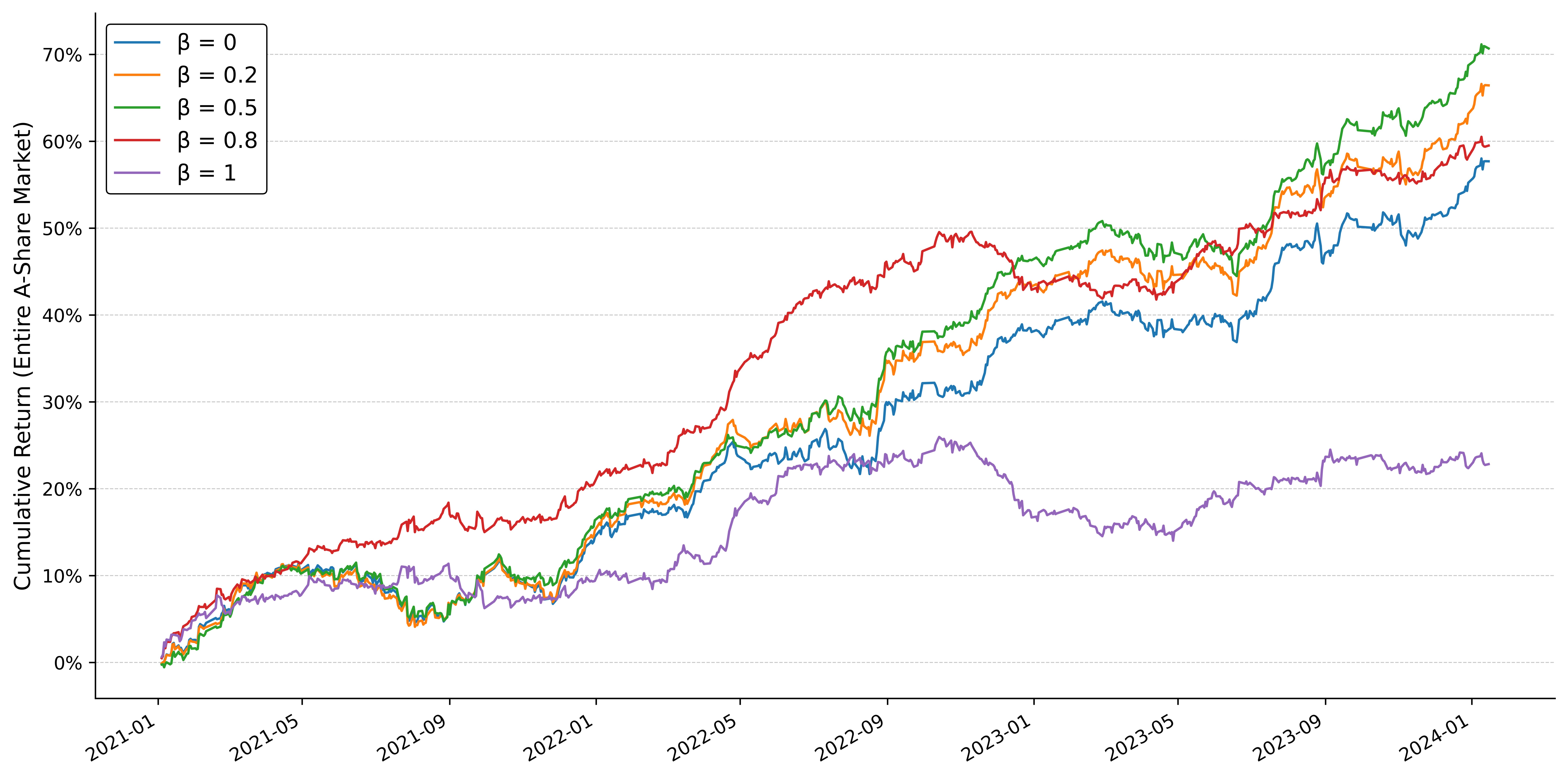}
    \caption{Cumulative return trajectories of portfolios constructed under different values of $\beta$ in the PPS formulation. The parameter $\beta$ balances predictive power against factor quality dimensions. Moderate values ($\beta = 0.5, 0.8$) yield superior performance compared to extreme values ($\beta = 0, 1$), suggesting that an appropriate trade-off between predictive accuracy and robustness leads to more stable and effective alpha selection.}
    \label{fig:sensitive_beta}
\end{figure*}

\begin{figure}[!h]
    \centering
    \includegraphics[width=1.\linewidth]{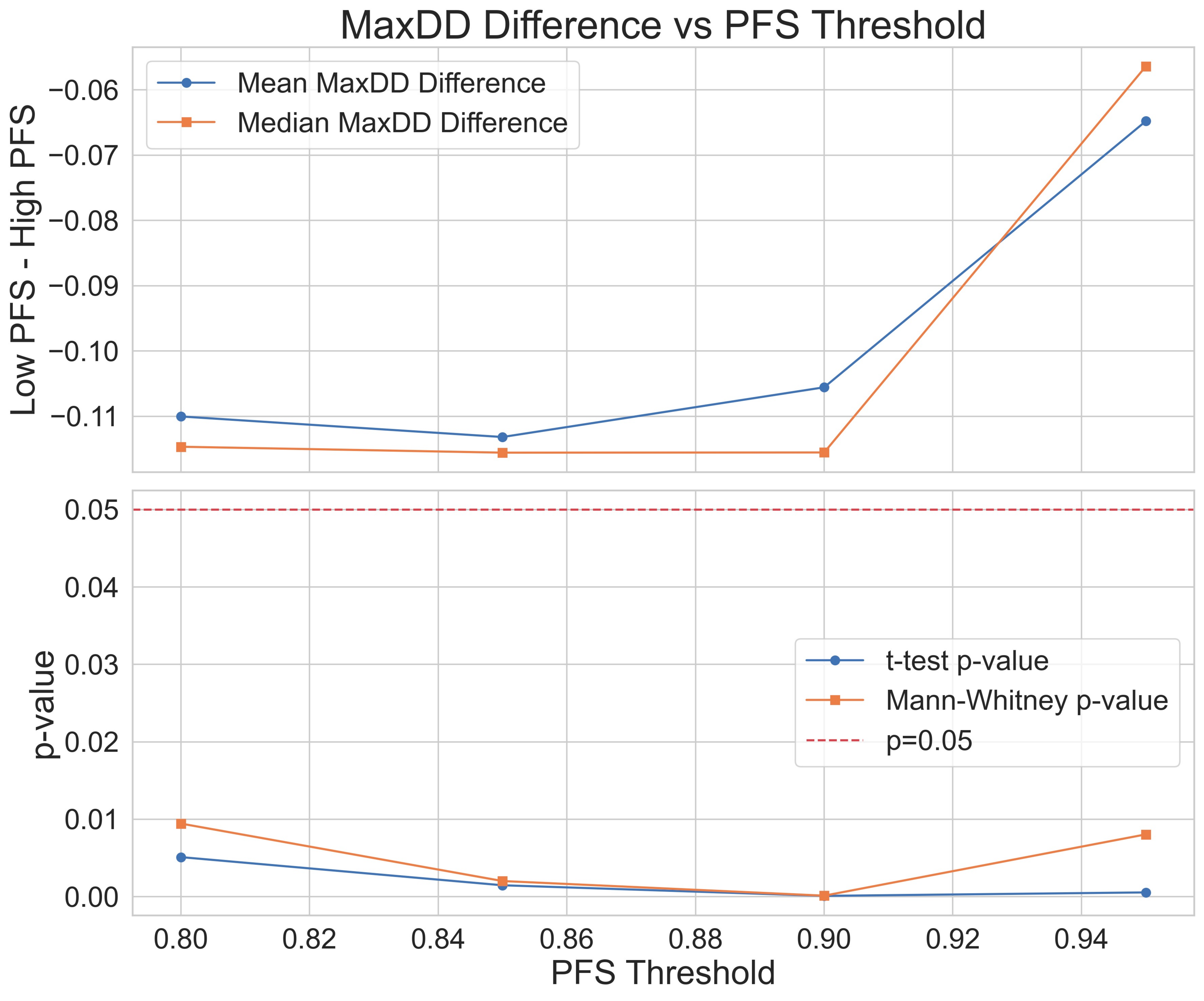}
    \caption{Sensitivity analysis of the PFS threshold on downside risk. The upper panel shows the mean and median differences in maximum drawdown (MaxDD) between low-PFS and high-PFS factor groups across different PFS thresholds. The lower panel reports the corresponding $p$-values from t-tests and Mann-Whitney U tests. A clear reduction in MaxDD is observed for higher-PFS groups, with statistical significance achieved when the PFS threshold ranges from 0.8 to 0.9. These results validate the effectiveness of PFS in identifying more stable factors under noise and support its utility in risk-aware factor selection.}
    \label{fig:sensitive_pfs}
\end{figure}

To answer \textbf{Q3}, sensitivity analyses were conducted on two key parameters: the weighting coefficient $\beta$ in the Predictive Power Score (PPS), and the threshold for the Perturbation Fidelity Score (PFS).

As shown in Figure~\ref{fig:sensitive_beta}, varying $\beta$ in the PPS formulation yields non-monotonic effects on cumulative portfolio returns. While moderate values (e.g., $\beta=0.5$ and $\beta=0.8$) lead to stronger performance, extreme values ($\beta=0$ or $\beta=1$) result in diminished returns. This pattern suggests that incorporating multiple factor quality dimensions—rather than relying solely on predictive power—can lead to more robust alpha selection and better generalization.

Figure~\ref{fig:sensitive_pfs} further examines the impact of the PFS threshold on downside risk. When comparing low-PFS and high-PFS factor groups, factors with higher PFS consistently exhibit lower maximum drawdown. The difference becomes statistically significant (p-value < 0.05) when the threshold lies between 0.8 and 0.9, confirming that PFS effectively captures stability under noise perturbation. Together, these findings demonstrate that the proposed metrics not only guide factor selection but also align with real-world risk control considerations.

\subsection{Justification of Evaluation Dimensions}
\label{sec:justification}

\begin{figure}[!h]
    \centering
    \includegraphics[width=1.0\linewidth]{fig/justify.pdf}
    \caption{\textbf{(a)}~Relative Rank Entropy (RRE) exhibits a strong negative correlation with annualized turnover (AnnTurn), suggesting that more temporally stable signals result in lower trading activity. \textbf{(b)}~Alphas with higher Perturbation Fidelity Score (PFS $\geq 0.9$) show significantly lower maximum drawdown (MaxDD), indicating greater robustness to market perturbations. \textbf{(c)}~The LLM-based logic score aligns closely with human expert rankings, as evidenced by consistently high NDCG@k values across different cutoff thresholds. \textbf{(d)}~AlphaEval significantly reduces evaluation time compared to traditional backtesting, highlighting its efficiency and scalability.}
    \label{fig:justify}
\end{figure}   

To answer \textbf{Q4}, we analyzed and verified the rationality of the newly proposed metrics:

\begin{itemize}
    \item \textbf{Temporal Stability vs. Turnover}: To assess the interpretability of Relative Rank Entropy (RRE), we examined its relationship with the annualized turnover rate (AnnTurn)\footnote{For definition of turnover, see the Appendix~\ref{appendix:basic}.}. As shown in Figure~\ref{fig:justify}(a), RRE exhibits a strong and statistically significant negative linear correlation with turnover. The fitted regression line indicates that as RRE increases—i.e., as alpha scores become more stable and structured—the resulting strategy becomes significantly less reactive, reflected in lower turnover rates. This suggests that RRE not only serves as a theoretical measure of rank consistency, but also reflects practical trading behavior.

    \item \textbf{Robustness vs. MaxDrawdown}: To evaluate the practical utility of PFS, we partitioned the alpha pool based on a threshold of $\mathrm{PFS} \ge 0.9$. As illustrated in Figure~\ref{fig:justify}(b), the high-PFS group exhibited substantially lower Max Drawdown (MaxDD), with a tighter and more favorable distribution. The definition of Max Drawdown is provided in Appendix~\ref{appendix:basic}. Statistical testing confirmed the difference to be highly significant. These results demonstrate that PFS is not only a predictive score but also an effective selection criterion for identifying robust, low-risk strategies.
    
    \item \textbf{Logic Score vs. Human Expert Judgment}: We compute $\text{NDCG@}k$ across multiple cutoff values ($k \in {5, 10, 20, 50, 100}$) to assess the model’s alignment with human rankings at different granularities~\cite{jarvelin_2002_dcg,wang_2013_ndcg}. As shown in Figure~\ref{fig:justify}(c), the model achieves consistently high NDCG scores, indicating strong agreement with human judgment in both top-ranked and overall alpha evaluation. These results further support the validity of the logic consistency dimension used in AlphaEval.
    
    \item \textbf{Diversity via Covariance Entropy}: The proposed Diversity Entropy (DH) provides a principled way to quantify the spread of variance across alpha signals by analyzing the spectrum of their covariance matrix. Intuitively, when multiple signals are highly correlated or collinear, their variance is concentrated along a few principal components, resulting in a low-entropy eigenvalue distribution. In contrast, when signals are orthogonal or capture complementary information, variance is more evenly distributed, yielding higher entropy. Therefore, DH serves as an effective proxy for the intrinsic dimensionality of the signal set, and can be used to detect and penalize multicollinearity. \footnote{Detailed proof is provided in the Appendix~\ref{appendix:dh_proof}.}
\end{itemize}

These results demonstrate that the AlphaEval metrics are not only computationally efficient but also aligned with real-world financial behaviors and theoretical intuition.

\subsection{Evaluation Efficiency}
To answer \textbf{Q5}, we compared the evaluation efficiency of AlphaEval with that of a traditional backtesting-based system. Unlike our metrics, portfolio backtesting exhibits path dependency. Positions, cash balances, and transaction costs evolve through state recursion. Dividing the time dimension into mutually exclusive weekly (or monthly) time slices and backtesting them independently causes cross-slice transfers (of positions, cash, and costs) to fail, leading to state leakage and non-additive gains/losses. Backtesting can be parallelized at the asset or parameter level, but sequential simulation of a single strategy along the timeline is essential to ensure accounting integrity.

For the parts of the computation that can be parallelized, we use 20 processes to parallelize the computation. As shown in Figure~\ref{fig:justify}(d), AlphaEval achieves a significant speedup, reducing relative evaluation time by more than 25\%. This improvement stems primarily from its backtesting-free design: all evaluation metrics in AlphaEval are formulated as functions that can be computed independently in parallel, in contrast to the inherently sequential nature of portfolio backtesting. This enables scalable and fast evaluations and accelerates the alpha mining process.

%% file: tex/6_conlusion.tex
\section{Conclusion}
\label{sec:conclusion}

This paper introduces \textbf{AlphaEval}, a unified, backtesting-free, and parallelizable framework for evaluating automated alpha-mining models. To address the limitations of conventional practice—most notably an over-reliance on backtests and the use of incomplete single-metric summaries—we propose five complementary dimensions that together provide a holistic view of alpha quality: predictive power, temporal stability, robustness to perturbations, financial logic, and diversity.

Extensive experiments spanning genetic programming, reinforcement learning, generative models, and large language models demonstrate the effectiveness of AlphaEval. The framework yields scores that align with backtesting outcomes while providing additional interpretability and diagnostic insight. Ablations confirm that the five dimensions are complementary, and empirical analyses show expected links with real-world behaviors such as turnover and drawdown. Moreover, AlphaEval substantially improves evaluation efficiency, enabling scalable alpha screening within large mining pipelines.

By decoupling evaluation from backtesting and releasing open-source tooling, AlphaEval promotes greater transparency, comparability, and reproducibility in alpha research. Looking forward, an important direction is to use AlphaEval not only as a post-hoc evaluator but also as a training signal during alpha generation—e.g., as reinforcement-learning rewards, differentiable surrogates, or prompt-conditioning targets. This would enable self-improving agents that optimize not only predictive performance but also stability, interpretability, and robustness. Extending AlphaEval to multi-frequency, multi-asset, and cross-market settings is another key avenue for future work.

\section{Limitations}
\label{sec:limit}

While AlphaEval is practical and methodologically transparent, it has several limitations. First, the framework evaluates \emph{alphas} rather than complete trading strategies; it does not replace simulation of position sizing, risk control, execution, or path-dependent accounting. Second, the \emph{Logic} dimension currently relies on LLM-based scoring and can be sensitive to model choices and prompt design; although we introduce calibration and cross-judge checks, residual evaluator bias may remain. Third, although the framework is, in principle, asset-class agnostic, our empirical study focuses on equities.

To broaden applicability, we are extending AlphaEval to futures and multi-asset settings. For equity index futures, we will use standard continuous-contract construction (e.g., volume/open-interest rolls), margin conventions, and cost models, and we will release instrument lists and rolling rules to ensure reproducibility and comparability across asset classes.

%% file: tex/7_appendix.tex
\section{Code}
Our code is available at \url{https://github.com/LeoDingggg/AlphaEval}.

\section{Basic Metrics}
\label{appendix:basic}
The features and operators used in the alpha mining process and their meanings are in Table~\ref{tab:operator}.
\begin{table*}[h]
    \centering
        \caption{Features and Operators and their corresponding meanings(partial list).}
    \vspace{-0.45cm}
    \begin{tabular}{cc}
    \toprule
     Names & Meaning  \\
     \midrule
    close        & Closing price of the trading day \\
    high         & Highest price during the trading day \\
    low          & Lowest price during the trading day \\
    open         & Opening price of the trading day \\
    volume       & Number of shares traded \\
     \midrule
      Add(x, y) & Element-wise addition of x and y \\
      Sub(x, y) & Element-wise subtraction of y from x \\
      Mul(x, y) & Element-wise multiplication of x and y \\
      Div(x, y) & Element-wise division of x by y \\
      Mean(x, d) & Time-series mean of x over the past d days \\
      Std(x, d) & Moving standard deviation of x over the past d days \\
      Corr(x, y, d) & Time-series Pearson correlation between x and y over the past d days \\
      Cov(x, y, d) & Time-series covariance between x and y over the past d days \\
    \bottomrule
    \end{tabular}
    \label{tab:operator}
\vspace{-0.3cm}
\end{table*}
For metrics used in previous work, the exact definition is given here.
\paragraph{ICIR:} Based on the definition of IC, ICIR can be defined as follows:
\begin{equation}
    \label{eq:icir}
    \mathrm{ICIR}=\frac{\mathrm{IC}}{\sqrt{\frac{1}{T-1}\sum_{i=1}^T(\mathrm{IC}_t-\mathrm{IC})^2}},
\end{equation}
RankICIR is defined similarly.
\paragraph{AR:} Given the final score matrix $S \in \mathbb{R}^{T \times N}$, a daily long-short portfolio is constructed at each time step $t$. Let $w_t \in \mathbb{R}^N$ denote the portfolio weights:
\begin{equation}
w_{t,n} =
\begin{cases}
+\frac{1}{K}, & \text{if } S_{t,n} \text{ is in the top-$K$}, \\
-\frac{1}{K}, & \text{if } S_{t,n} \text{ is in the bottom-$K$}, \\
0, & \text{otherwise},
\end{cases}
\end{equation}

where $K$ is the number of assets selected for both long and short sides. The portfolio return at time $t$ is defined as:

\begin{equation}
r_t^{\Delta T} = w_t^\top y_t^{\Delta T}.
\end{equation}

Therefore, Annualized Return (AR) can be defined as:
\begin{equation}
\mathrm{AR} = \bar{r} \cdot D,
\end{equation}
where $\bar{r} = \frac{1}{T} \sum_{t=1}^{T} r_t$ is the average daily return, and $D$ denotes the number of trading days in a year (typically $D = 252$).

\paragraph{SR:} Sharpe Ratio (SR) measures the performance of an investment compared to a risk-free asset, after adjusting for its risk. The defination is follows:
\begin{equation}
\mathrm{SR} = \frac{\bar{r}}{\sigma_r} \cdot \sqrt{D},
\end{equation}
where $\sigma_r = \sqrt{ \frac{1}{T - 1} \sum_{t=1}^{T} (r_t - \bar{r})^2 }$ is the standard deviation of daily returns. The risk-free rate is assumed to be zero for simplicity.

\paragraph{Turnover Rate:} Turnover Rate measures the trading frequency implied by the signal. A higher turnover indicates more frequent portfolio rebalancing, which may lead to increased transaction costs and lower net returns in real-world deployment.
It can be defined as:
\begin{equation}
\mathrm{Turn} = \frac{1}{T - 1} \sum_{t=2}^{T} | w_t - w_{t-1} |_1,
\end{equation}
where $| \cdot |_1$ denotes the $\ell_1$ norm, measuring the total absolute change in position weights across consecutive time steps. The range of the turnover rate is $[0, 2]$.

\paragraph{MDD:} Maximum Drawdown (MaxDD) quantifies the worst-case loss from a historical peak in the cumulative return curve. It reflects the risk of large interim losses and is widely used as a measure of downside risk in portfolio evaluation. Let the cumulative net asset value (NAV) series be defined as:

\begin{equation}
\mathrm{NAV}t = \prod{i=1}^{t} (1 + r_i).
\end{equation}

Then the maximum drawdown is computed as:

\begin{equation}
\mathrm{MaxDD} = \max_{t \in [1,T]} \left( \max_{s \in [1,t]} \frac{\mathrm{NAV}_s - \mathrm{NAV}_t}{\mathrm{NAV}_s} \right),
\end{equation}
representing the largest observed loss from a peak to a subsequent trough in the NAV curve.

\paragraph{NDCG@k.} Let $\mathrm{rel}_i$ denote the relevance grade of the item at rank $i$ in a predicted ordering. Define
\[
\mathrm{DCG}@k \;=\; \sum_{i=1}^{k} \frac{2^{\,\mathrm{rel}_i}-1}{\log_2(i+1)}, 
\qquad
\mathrm{NDCG}@k \;=\; \frac{\mathrm{DCG}@k}{\mathrm{IDCG}@k},
\]
where $\mathrm{IDCG}@k$ is the maximum possible DCG@k achieved by the ideal (ground-truth) ordering. We report NDCG@k to quantify how closely the predicted ordering (e.g., LLM or human-judged Logic scores) matches the reference ordering.

\section{Datasets}
\label{appendix:dataset}
We get the A-share dataset and S\&P 500 dataset from Qlib. The time ranges for Train/Valid/Test on A-Share dataset are 2010-01-01 -- 2019-12-31/2020-01-01 -- 2020-12-31/2021-01-01 -- 2024-12-31. The time ranges for Train/Valid/Test on S\&P 500 dataset are 2010-01-01 -- 2015-12-31/2016-01-01 -- 2016-12-31/2017-01-01 -- 2020-12-31.

\section{Feature \& Operator}
See Tabel~\ref{tab:operator}.

\section{Significance Analysis}

In this section, we analyze the significance of the linear fitting experiments of RRE to the annualized turnover rate in the main text, as well as the experiments comparing the threshold of PFS to MaxDD.

Table~\ref{tab:rre-turnover} shows the results of an OLS regression predicting annualized turnover using RRE as the sole independent variable. The coefficient for RRE is significantly negative ($\beta$ = -4.361, $p <$ 0.001), indicating that higher RRE scores are strongly associated with lower turnover rates. The model explains 81.5\% of the variance in turnover ($R^2$ = 0.815), suggesting that RRE can serve as a meaningful proxy for trading intensity.

\begin{table}
\centering
\caption{OLS Regression Results: RRE as Predictor of Annualized Turnover}
\label{tab:rre-turnover}
\vspace{-0.45cm}
\begin{tabular}{lcc}
\toprule
\textbf{Variable} & \textbf{Coefficient} & \textbf{Std. Error} \\
\midrule
Intercept         & 4.494***             & 0.104 \\
RRE               & -4.361***            & 0.113 \\
\midrule
\textbf{Model Statistics} & & \\
$R^2$             & 0.815 & \\
\bottomrule
\multicolumn{3}{l}{\footnotesize{*** p $<$ 0.001}}
\end{tabular}
\vspace{-0.3cm}
\end{table}
\begin{table}
\centering
\caption{Statistical Comparison of MaxDD between High and Low PFS Groups (Threshold = 0.9)}
\vspace{-0.45cm}
\label{tab:pfs-significance}
\begin{tabular}{lccc}
\toprule
\textbf{Test} & \textbf{Statistic} & \textbf{p-value} \\
\midrule
T-test     & $t$ = 4.122         & 0.0001 \\
Mann-Whitney U test            & $U$ = 11,381.5   & 0.0001 \\
\bottomrule
\end{tabular}
\vspace{-0.3cm}
\end{table}

\begin{table*}
    \centering
    \caption{The main results on US Market (S\&P 500) under the AlphaEval framework. \textbf{Bold} is the highest, \underline{underlined} is the second, random is only used as a reference value and is not involved in the comparison.}
    \vspace{-0.45cm}
    \label{tab:main-results-us}
    \begin{tabular}{ccccccc}
        \toprule
        \multicolumn{2}{c}{Method} & Predictive$\uparrow$ & Stability$\uparrow$ & Robustness$\uparrow$ & Diversity$\uparrow$ & Logic$\uparrow$ \\
        \midrule
        &Random & 0.006 & 0.962 & 0.899 & 0.976 & 63.0\\
        \midrule
        \multirow{3}{*}{GA-Based}&GP & 0.016 & 0.743 & \underline{0.978} & \underline{0.924} & 64.0 \\
        &AutoAlpha  & \underline{0.023} & 0.838 & 0.753 & 0.572 & 66.0 \\
        \midrule
        \multirow{2}{*}{RL-Based}&AlphaGen & 0.010 & 0.770 & 0.584 & 0.425 & 58.0 \\
        &AlphaQCM & 0.011 & \textbf{0.995} & \textbf{0.994} & 0.687 & 61.0 \\
        \midrule
        GANs-Based&AlphaForge & 0.019 & \underline{0.839} & 0.949 & \textbf{0.994} & 62.5 \\
        \midrule
        \multirow{2}{*}{LLMs-Based}&FAMA & 0.021 & 0.804 & 0.961 & 0.79 & \underline{70.0}  \\
        &AlphaAgent & \textbf{0.025} & 0.813 & 0.464 & 0.756 & \textbf{71.5} \\   
        \bottomrule
    \end{tabular}
    \vspace{-0.3cm}
\end{table*}

To assess the effectiveness of PFS as a filtering criterion for strategy robustness, we compared Max Drawdown (MaxDD) between groups with PFS $\geq$ 0.9 and PFS $<$ 0.9. As reported in Table~\ref{tab:pfs-significance}, the difference in MaxDD between the two groups was statistically significant. The independent samples t-test yielded $t$(70.59) = 4.12, $p$ = 0.0001, while the non-parametric Mann-Whitney U test also confirmed this difference ($U$ = 11,381.5, $p$ = 0.0001). These results indicate that higher PFS scores are strongly associated with lower drawdowns, validating the practical utility of PFS as a reliable risk-screening metric.

\section{Extra Results}
\label{appendix:er}

We evaluate all models except AlphaEvlove using AlphaEval on the S\&P 500 dataset. The results are in Table~\ref{tab:main-results-us}.

\section{Proof of DH as a Multicollinearity Detector}
\label{appendix:dh_proof}

We formally justify that the proposed \textbf{Diversity Entropy (DH)} is inversely related to the degree of multicollinearity among alpha signals.

\begin{theorem}
Let $\{S^{(i)}\}_{i=1}^m$ be $m$ alpha signals represented as column vectors of equal length, and let $C \in \mathbb{R}^{m \times m}$ be their sample covariance matrix. Let $\lambda_1, \dots, \lambda_m$ be the eigenvalues of $C$, and define the normalized spectral distribution as $p_i = \frac{\lambda_i}{\sum_j \lambda_j}$. The Diversity Entropy is defined as:
\[
\mathrm{DH} = \frac{-\sum_{i=1}^m p_i \log p_i}{\log m}.
\]
Then:
\begin{enumerate}
    \item $\mathrm{DH} \in [0, 1]$.
    \item $\mathrm{DH} = 0$ if and only if the rank of $C$ is 1 (i.e., all signals are perfectly collinear).
    \item Smaller values of $\mathrm{DH}$ imply stronger linear dependence among the signals (i.e., higher multicollinearity).
\end{enumerate}
\end{theorem}

\begin{proof}
\textbf{(1)} The set $\{p_i\}$ forms a valid probability distribution since $\lambda_i \ge 0$ and $\sum_i p_i = 1$. The normalized Shannon entropy of any such distribution lies in $[0, 1]$.

\textbf{(2)} If $\mathrm{rank}(C) = 1$, then $C$ has only one non-zero eigenvalue, say $\lambda_1 > 0$ and $\lambda_{2}, \dots, \lambda_m = 0$. Then $p_1 = 1$ and $p_i = 0$ for $i > 1$, so $\mathrm{DH} = 0$.

Conversely, if $\mathrm{DH} = 0$, then the entropy of $\{p_i\}$ is zero, which only occurs when all the probability mass is concentrated at a single index, i.e., $p_i = 1$ for some $i$, implying all other eigenvalues are zero. Thus $C$ is rank 1 and signals are linearly dependent.

\textbf{(3)} When multicollinearity is strong, the signal vectors lie close to a low-dimensional subspace, and $C$ becomes ill-conditioned or low-rank. Its eigenvalue spectrum becomes more skewed (i.e., concentrated), reducing entropy. Thus, lower DH reflects stronger redundancy among signals.
\end{proof}

\section{Details of LLMs}
\label{appendix:llm}
In this paper, for the implementation of FAMA and AlphaAgent, we use GPT-4o as the base model. For LLM used in FAMA, we set the $\text{max\_tokens}=500,\text{temperature=0.5}$. For Idea Agent in AlphaAgent, we set the  $\text{temperature=1.0}$. For Factor Agent in AlphaAgent, we set the $\text{temperature=0.3}$. For Eval Agent in AlphaAgent, we set the $\text{temperature=0.4}$.

For the implementation of Logic Score, we use GPT-4o to judge the logic of alpha. We set the $\text{max\_tokens}=1000,\text{temperature=0.2}$, the prompt is following:
\begin{quote}
    Below is a set of quantitative factor expressions designed using qlib syntax.
    
    Please score each factor from 50 to 100 based on the rationality of financial market logic (full score), and provide the corresponding logical explanation.
    
    When scoring, differences in scores can be larger: logical factors can receive very high scores, and vice versa.
    
    Factor list: \{factor\_expressions\}
    
    Please return \textbf{a pure JSON array only}, without any Markdown code blocks. "
    
    The array length should match the factor list, and each element should be an object containing:
    \begin{itemize}
        \item factor: the factor expression
        \item score: numeric score (50–100)
        \item explanation: a brief logical explanation
    \end{itemize}
\end{quote}